\documentclass[11pt]{article}
\usepackage[numbers, compress]{natbib}
\usepackage[margin=1in]{geometry}
\usepackage{subcaption}

\usepackage[table]{xcolor}  % colors
\definecolor{Gray}{gray}{0.9}
\definecolor{midgreen}{rgb}{0.1,0.5,0.1}
\definecolor{darkgray}{gray}{0.25}
\definecolor{lightblue}{rgb}{0.25,0.25,0.8}
\definecolor{mydarkblue}{rgb}{0,0.08,0.45}
\usepackage[colorlinks,
linkcolor=lightblue,
anchorcolor=blue,
urlcolor=mydarkblue,
citecolor=lightblue]{hyperref}
\usepackage{makebox}
\usepackage{amsmath}
\usepackage{amssymb}
\usepackage{mathtools}
\usepackage{amsthm}
\usepackage{setspace}
\usepackage{bm}
\usepackage{bbm}

\newcommand{\norm}[1]{\ensuremath{\left\| #1 \right\|}}

\newcommand{\abs}[1]{\left |#1\right|}

\def\0{{\bm 0}}

\def\c{{\bm c}}

\def\k{{\bm k}}

\def\q{{\bm q}}
\def\r{{\bm r}}

\def\v{{\bm v}}

\def\x{{\bm x}}
\def\y{{\bm y}}

\def\I{{\bm I}}
\def\J{{\bm J}}
\def\K{{\bm K}}

\def\M{{\bm M}}

\def\S{{\bm S}}

\def\V{{\bm V}}

\def\X{{\bm X}}

\def\bpsi{\boldsymbol{\psi}}

\def\bphi{\boldsymbol{\phi}}

\def\E{\mathop{{\mathbb{E}}}}
\def\R{{\mathbf{R}}}

\def\Ncal{\mathcal{N}}

\def\RR{\mathbb{R}}

\newtheorem{theorem}{Theorem}
\newtheorem{lemma}{Lemma}

\newtheorem{fact}{Fact}

\newtheorem{definition}{Definition}

\def\llama{$\mathtt{Llama}$-$\mathtt{3.1}$-$\mathtt{8B}$-$\mathtt{Instruct}$}
\def\mistral{$\mathtt{Ministral}$-$\mathtt{8B}$-$\mathtt{Instruct}$}

\def\whK{\widehat{\makebox*{$m$}{$ \K $}}}
\def\whV{\widehat{\makebox*{$m$}{$ \V $}}}
\def\bPsi{{\boldsymbol \Psi}}
\def\bpsi{{\boldsymbol \psi}}
\def\bfpn{b_{\text{FPN}}}

\newcommand{\email}[1]{\href{mailto:#1}{\color{black} \texttt{#1}}}

\title{PolarQuant: Quantizing KV Caches with Polar Transformation}

\author{
  Insu Han \\
  KAIST \\
  \email{insu.han@kaist.ac.kr} \\
  \and
  Praneeth Kacham \\
  Google Research\\
  \email{pkacham@google.com} \\
  \and 
  Amin Karbasi \\
  Yale University\\
  \email{amin.karbasi@yale.edu} \\
  \and 
  Vahab Mirrokni \\
  Google Research\\
  \email{mirrokni@google.com} \\
  \and 
  Amir Zandieh \\
  Google Research\\
  \email{zandieh@google.com} \\
}

\date{}

\usepackage[capitalize]{cleverref}
\usepackage{algorithm}
\usepackage{algorithmic}
\usepackage{amsthm}

\usepackage{multirow,makecell}
\usepackage{booktabs}

\usepackage{tikz}
\usetikzlibrary{matrix}
\usetikzlibrary{positioning}

\begin{document}

\maketitle
% \footnote{Equal Contribution}
\def\thefootnote{*}\footnotetext{
Authors are listed alphabetically.}

\begin{abstract}
Large language models (LLMs) require significant memory to store Key-Value (KV) embeddings in their KV cache, especially when handling long-range contexts. 
Quantization of these KV embeddings is a common technique to reduce memory consumption. 
This work introduces PolarQuant, a novel quantization method employing random preconditioning and polar transformation. 
Our method transforms the KV embeddings into polar coordinates using an efficient recursive algorithm and then quantizes resulting angles. 
Our key insight is that, after random preconditioning, the angles in the polar representation exhibit a tightly bounded and highly concentrated distribution with an analytically computable form. 
This nice distribution eliminates the need for explicit normalization, a step required by traditional quantization methods which introduces significant memory overhead because quantization parameters (e.g., zero point and scale) must be stored in full precision per each data block. 
PolarQuant bypasses this normalization step, enabling substantial memory savings. 
The long-context evaluation demonstrates that PolarQuant compresses the KV cache by over $\mathbf{\times 4.2}$ while achieving the best quality scores compared to the state-of-the-art methods.
\end{abstract}

\section{Introduction}
\label{intro}

Transformer-based models form the backbone of modern artificial intelligence systems and have been instrumental in driving the ongoing AI revolution.
Their applications span various domains, including frontier language models (LLM)~\cite{achiam2023gpt,claude, Gemini} to text-to-image~\cite{ramesh2022hierarchical,firefly,midjourney}, text-to-video synthesis~\cite{Veo, sora}, coding assistants~\cite{copilot} and even multimodal models that ingest text, audio, image, and video data~\cite{gpt4o, Gemini}. 
The self-attention mechanism~\cite{vaswani2017attention} is at the heart of these models as it enables capturing the direct dependencies of all tokens in the input sequence.
The ability of these models grows along with their size and context length~\cite{kaplan2020scaling}, which leads to computational challenges in terms of huge memory consumption to support fast inference.

Most large language models, as well as multimodal and video models, adopt an autoregressive, decoder-only architecture that generates tokens sequentially. 
To avoid redundant attention score computations during the generation phase, these models employ a KV caching scheme, which stores the key and value embeddings of previously generated tokens in each attention layer. 
However, a significant challenge in deploying autoregressive Transformers lies in the substantial memory demands, as the KV cache size scales with both the model size (i.e., the number of layers and attention heads) and the context length. 
Furthermore, serving each model session typically necessitates its own dedicated KV cache, further compounding memory demands.
This has become a significant bottleneck in terms of memory usage and computational speed, particularly for models with long context lengths. 
Thus, reducing the KV cache size while preserving accuracy is critical to addressing these limitations.

Several approaches have been proposed to address the KV caching challenge. 
Architectural solutions, such as multi-query attention \cite{shazeer2019fast}, grouped-query attention \cite{ainslie2023gqa}, and multi-head latent attention~\cite{dai2024deepseekmoe}, modify the transformer architecture to reduce the memory demands during inference by decreasing the \emph{number} of key-value pairs that are to be stored.

Another orthogonal line of research focuses on reducing the KV cache size by pruning or evicting redundant or unimportant tokens~\cite{beltagy2020longformer, zhang2024h2o, liu2024scissorhands, xiao2023efficient, zandieh2024subgen, li2024snapkv}. However, eviction strategies face limitations in long-context tasks that require precise knowledge extraction, such as needle-in-haystack scenarios.
Additionally, some recent works tackle the issue from a systems perspective, such as offloading~\cite{sheng2023flexgen, sun2024shadowkv} or integrating virtual memory and paging strategies into the attention mechanism~\cite{kwon2023efficient}.

A simple yet effective approach to reducing KV cache size is quantizing the floating-point numbers (FPN) in the KV cache by storing their approximations using fewer number of bits. 
Several quantization methods have been proposed specifically for the KV cache~\cite{yue2024wkvquant, yang2024no, dong2024qaq, kang2024gear, zhang2024kv, liu2024kivi, hooper2024kvquant}.
Recently, a new KV cache quantization method called QJL~\cite{zandieh2024qjl} introduced an efficient, data-oblivious 1-bit quantization approach based on sketching techniques.
This method does not require tuning or adaptation to the input data, incurs significantly lower memory overhead compared to prior works, and achieves superior performance.
A very recent work, Lexico~\cite{kim2024lexico}, applies techniques from sparse representation learning to compress the KV cache by learning a universal dictionary such that all key and value embeddings are represented as extremely sparse vectors within the learned dictionary. 
Unfortunately, this approach requires solving a computationally expensive matching pursuit algorithm for each key and value embedding, making Lexico relatively slow.

Traditional KV cache quantization methods face significant ``memory overhead'' due to the need for data normalization before quantization. 
Most methods group data into blocks--either channel-wise or token-wise--and independently normalize each block which requires computing and storing quantization constants (e.g., zero points and scales) in full precision. 
This process can add over 1 additional bit per quantized number, resulting in considerable memory overhead.
We show that applying a random preconditioning matrix on the embedding vectors eliminates the need for data normalization. 
This approach aligns with the recent use of random Hadamard matrices as preconditioners before quantizing embedding vectors in attention layers to improve quality~\cite{shah2024flashattention, ashkboos2024quarot}.

\subsection{Contributions}
We propose quantizing KV vectors in polar coordinates instead of the usual Cartesian coordinates. This shift enables more efficient representation and compression of KV embeddings.

{\bf Random Preconditioning.} 
We apply a random rotation to the vectors before quantization, which preserves inner products while randomizing the distribution of each vector. 
This preconditioning causes the angles in polar coordinates to concentrate, allowing us to quantize them with high precision using small bit-widths.
We derive the analytical distribution of angles after preconditioning and leverage this insight to construct an optimized quantization codebook, minimizing quantization error.

{\bf Recursive Polar Transformation.} 
We introduce a computationally efficient recursive polar transformation that converts vectors into polar coordinates, enabling practical deployment of our approach.
We are able to prove an error bound in \cref{main_thrm} showing our algorithm is asymptotically optimal for worst-case KV embedding vectors.

{\bf Performance on Long-Context Tasks.} We evaluate PolarQuant on long-context tasks and demonstrate that it achieves the best quality scores compared to competing methods while compressing the KV cache memory by over $\mathbf{\times 4.2}$.

\section{Preliminaries} \label{sec:prelim}
We use boldface lowercase letters, such as $\x$ and $\y$, to denote vectors, and boldface uppercase letters, like $\M$, to denote matrices.
To denote a slice of a vector $\x$ between the coordinate indices $i$ and $j$ inclusive of the endpoints, we use the notation $\x_{i:j}$. For a matrix $\M$, we write $\M_{i,:}$ to denote its $i$-th row vector, which we will simply refer to as $\M_i$.

\subsection{Efficient Token Generation and KV Caching}
Autoregressive Transformers often utilize cache storage for faster token generation. 
Given an input prompt, models encode the prompt information into two types of embeddings, called Key and Value. 
To generate subsequence tokens efficiently, the Key-Value (KV) embeddings are cached to avoid recomputing them. 

The Key-Value (KV) caching method leverages the architecture of transformer decodcers, where a causal mask in applied in the attention mechanism. 
Once the keys and values are computed for a given token, they remain unchanged for subsequent token generation. 
By caching these key-value pairs, the model avoids redundant computations, as it only needs to compute the query for the current token and reuse the cached keys and values for attention.

This approach significantly reduces computation time during token generation. 
Instead of processing the entire sequence repeatedly, the KV cache enables the model to efficiently focus on the incremental computation of new tokens. 
This makes the method particularly useful in real-time applications, such as conversational AI and text generation, where fast and resource-efficient inference is critical.

\subsection{Random Preconditioning}
A critical step in the PolarQuant algorithm is random preconditioning of the KV vectors prior to quantization.
This involves applying a random projection matrix to the embedding vectors before quantizing them.
To analyze the algorithm effectively, we rely on specific facts and properties of multivariate normal random variables, which are outlined below.

\begin{figure*}[t!]
    % \vspace{-0.02in}
    \centering
    \includegraphics[width=0.9\textwidth]{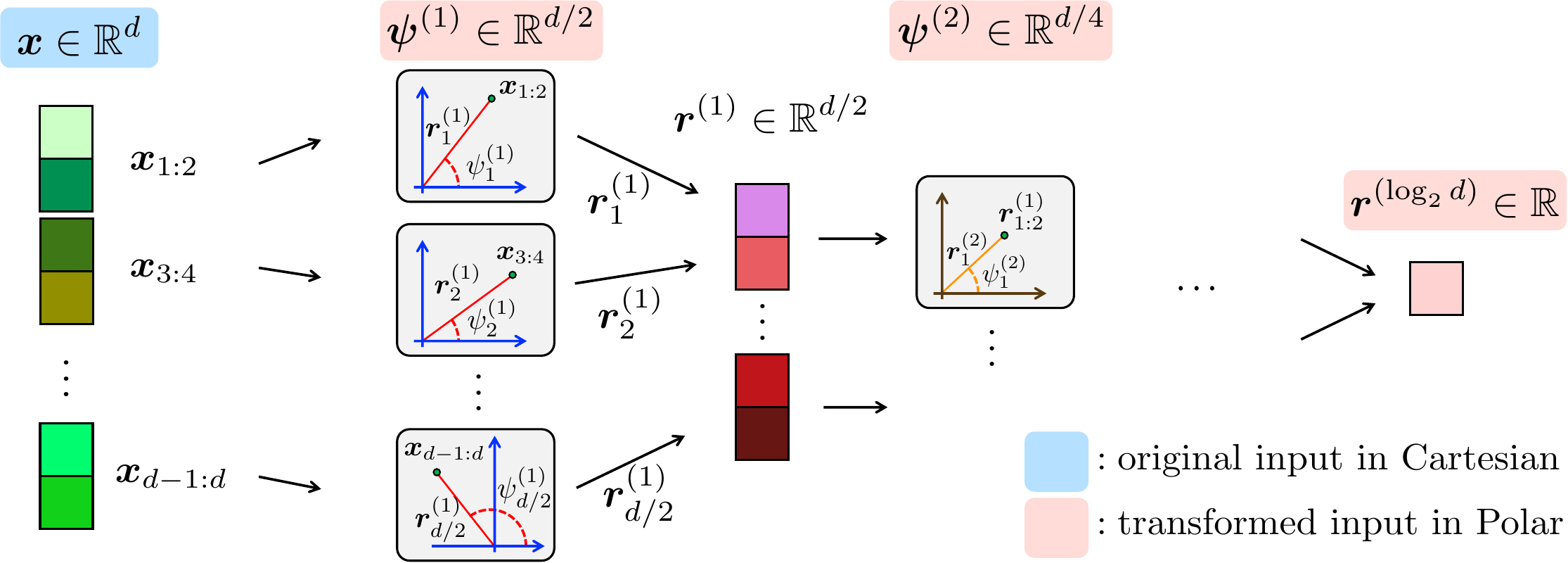}
    % \vspace{-0.1in}
    \caption{Overview of recursive polar transformation procedure in \cref{def_cartesian_to_polar}}\label{fig_polar_transform_diagram}
\end{figure*}

\begin{fact}\label{lem_norm_gaussian}
    For any positive integer $d$, if $\x \in \RR^d$ is a zero mean unit variance isotropic Gaussian random variable in dimension $d$, i.e., $\x \sim \Ncal \left( 0, \I_d \right)$, then its $2$-norm, denoted by $r := \| \x \|_2$, follows a generalized gamma distribution with the following probability density for any $r \ge 0$:
    \[
    f_R(r) = \frac{2}{2^{d/2} \cdot \Gamma(d/2)} r^{d-1} \exp\left( -r^2/2 \right)
    \]
\end{fact}

The proof of Fact~\ref{lem_norm_gaussian} is provided in \cref{sec:proof_lem_norm_gaussian}.
% \begin{proof}
%     The cumulative distribution function (c.d.f) of the random variable $R$ can be computed as follows:
%     \begin{align*}
%     F_R(r) &:= \Pr_{\x}[\norm{\x}_2 \le r]\\ 
%     &= \Pr_{\x}[\norm{\x}_2^2 \le r^2] = \frac{1}{\Gamma(d/2)} \gamma(d/2, r^2/2), 
%     \end{align*}
%     where the last equality is because the squared norm of $\x$, by definition, is Chi-squared random variable.
%     Differentiating the above c.d.f gives us the p.d.f of $R$:
%     \[
%     f_R(r) = \frac{2}{2^{d/2} \cdot \Gamma(d/2)} r^{d-1} \exp\left( -r^2/2 \right).
%     \qedhere\]
% \end{proof}
We also use the following facts about the moments of the univariate normal distribution.
\begin{fact}[Moments of Normal Random Variable]\label{moment_normal}
    If $x$ is a normal random variable with zero mean and unit variance $x \sim \Ncal(0, 1)$, then for any integer $\ell$, $\E_{x \sim \Ncal(0,1)} \left[ |x|^\ell \right] = 2^{\ell/2} \Gamma((\ell+1)/2) / \sqrt{\pi}$.
\end{fact}

PolarQuant algorithm applies a random preconditioning prior to quantization. 
This preconditioning involves multiplying each embedding vector by a shared random sketch matrix $\S$ with i.i.d. normal entries. 
By the Johnson-Lindenstrauss (JL) lemma~\cite{dasgupta2003elementary}, this preconditioning preserves the norms and inner products of the embedding vectors with minimal distortion. 
A key property of this preconditioning, which we will leverage in our later analysis, is that the embedding vectors after preconditioning follow a multivariate normal distribution.
This is formalized in the following fact.

\begin{fact}\label{fact_dist_vector_post_random_sketching}
    For any vector $\x \in \RR^d$ if $\S \in \RR^{m \times d}$ is a random matrix with i.i.d. normal entries $\S_{i,j} \sim \Ncal(0, 1)$, then the vector $\S \cdot \x$ has multivariate normal distribution $\S \cdot \x \sim \Ncal\left( 0, \norm{\x}_2  \cdot \I_m \right)$.
\end{fact}

The following lemma establishes the distribution of the polar angle of a point $(x, y)$ in dimension 2, where the $x$ and $y$ coordinates are independent samples from the Euclidean norm of multivariate normal random variables.
\begin{lemma}
    For any positive integer $d$, if $x, y \ge 0$ are two i.i.d. random variables with generalized gamma distribution with probability density function $f_Z(z) = \frac{2}{2^{d/2} \cdot \Gamma(d/2)} z^{d-1} \exp\left( -z^2/2 \right)$, then the angle variable $\theta := \tan^{-1}(y / x)$ follows the probability density function:
    \[
    f_{\Theta}(\theta) = \frac{\Gamma(d)}{2^{d - 2} \cdot \Gamma(d/2)^2} \cdot \sin^{d-1}(2\theta).
    \]
Additionally, $\E[\Theta] = \pi/4$ and $\mathrm{Var}(\Theta) = O(1/\sqrt{d})$.
\end{lemma}
See \cref{sec:lem_quant_error} for a proof.

\section{PolarQuant}

We now describe our approach of quantizing angles in polar coordinates and using it to the KV cache problem. In \cref{sec:polar}, we introduce how to recursively transform Cartesian vector to polar coordinates. In \cref{sec:angle_distribution}, we provide an analysis of polar angle distributions with preconditioning. In \cref{sec:polar-quant}, we explain details of quantization polar transformed embeddings and practical implementation. 

\subsection{Recursive Polar Transformation} \label{sec:polar}

There are various methods to derive the polar representation of $\RR^d$. 
Here we propose a polar transformation that can be recursively computed from the Cartesian coordinates of points in $\RR^d$. Throughout this work, we assume that $d$ is an integer power of $2$.

At a high level, our approach begins by grouping pairs of coordinates of a $d$-dimensional vector $\x$ and transforming each pair into 2D polar coordinates. This produces $d/2$ radius and angle pairs. 
%We call this to level $1$ transform. 
Next, we gather $d/2$ of radii and apply the polar transform to them. This procedure is recursively repeated $\log_2d$ times and the final output consists of a single final radius and a collection of $1, 2, 4, \dots, d/2$-dimensional angle vectors. A formal definition is provided in \cref{def_cartesian_to_polar}. 

\begin{definition}[Cartesian to Polar Transformation]\label{def_cartesian_to_polar}
    For any integer power of two $d$, the polar representation of any vector $\x \in \RR^d$ includes $d-1$ angles and a radius. 
    Angles are organized into a collection of $\log_2d$ vector of angles $\psi^{(1)}, \psi^{(2)}, \ldots \psi^{(\log_2d)}$ such that $\psi^{(1)} \in [0, 2\pi)^{d/2}$ and $\psi^{(\ell)} \in [0, \pi/2]^{d/2^\ell}$ for any $\ell \ge 2$. 
    In other words, the angles are computed in $\log_2d$ levels and there are $d/2^\ell$ angles in level $l$.
    These angles are defined by the following relation for $\ell \in \{2, 3, \ldots \log_2d\}$:
    \begin{align*}
    \psi^{(1)}_{j} &:= \tan^{-1}\left( \x_{2j} / \x_{2j-1} \right) ~~\text{ for } j \in [d/2], \\
    \psi^{(\ell)}_{j} &:= \tan^{-1}\left( \frac{\norm{\x_{(j-1/2)2^\ell+1:j2^\ell}}_2 }{ \norm{\x_{(j-1)2^\ell+1:(j-1/2)2^\ell}}_2 } \right) ~~\text{ for } j \in[d/2^\ell].
    \end{align*}
    The reverse of this transformation maps the angles and the radius of any point to its Cartesian vector representation using the following equation:
    \begin{align*}
    \x_i &= \norm{\x}_2 \cdot \prod_{\ell=1}^{\log_2d} \left(\cos\psi^{(\ell)}_{\lfloor \frac{i}{2^\ell} \rfloor} \right)^{\bm{1}_{\{ (i~\text{mod}~2^\ell) \le 2^{\ell-1} \}}} \cdot \prod_{\ell=1}^{\log_2d} \left( \sin\psi^{(\ell)}_{\lfloor \frac{i}{2^\ell} \rfloor} \right)^{\bm{1}_{\{ (i~\text{mod}~2^\ell) > 2^{\ell-1} \}}}
    \end{align*}
\end{definition}

% Our approach recursively produces angles in $\log_2d$ levels. It begins by grouping pairs of coordinates of a $d$-dimensional vector $\x$ and transforming each pair into 2D polar coordinates. This step produces $d/2$ angles in level $1$.
% Each of the angles in the first level are in the range $[0, 2\pi)$, denoted by $\psi^{(1)}$ in \cref{def_cartesian_to_polar}.
% The process then recurses on the radii derived from the paired coordinates in the previous level, with the key difference that, since the radii are strictly positive, all subsequent angles lie within $[0, \pi/2]$. The algorithm continues for $\log_2d$ iterations, ultimately generating a total of $d-1$ angles.

A visual diagram of the algorithm is shown in \cref{fig_polar_transform_diagram} and the pseudocode is provided in \cref{alg:polar-quant}
(see \textsc{Polar} procedure). In what follows, we analyze the distribution of angles generated in each quantization level.

\subsection{Distribution of Polar Angles Under Random Preconditioning} \label{sec:angle_distribution}

One of our primary objectives is to eliminate the need for explicit normalization (e.g., minimum/maximum values) of the KV cache data prior to quantization, thereby reducing quantization overhead. 
To achieve this, our algorithm applies random preconditioning to the embedding vectors.
This preconditioning involves multiplying each embedding vector by a shared random sketch matrix $\S$ with i.i.d. normal entries. 
By the Johnson-Lindenstrauss (JL) lemma~\cite{dasgupta2003elementary}, this preconditioning preserves the norms and inner products\footnote{For our implementation, we use random rotation matrices (square matrices $P$ satisfying $P^{\top} P = I$), which preserve the norms and inner products exactly while removing the independence across projected coordinates which we use for our theoretical results.} of the embedding vectors with minimal distortion. 
A key property of this preconditioning, which we will leverage in our later analysis, is that the embedding vectors after preconditioning follow a multivariate normal distribution.
This has been formalized in Fact~\ref{fact_dist_vector_post_random_sketching}.

During the preconditioning stage, the sketch is applied to all embedding vectors in the KV cache, allowing the analysis of PolarQuant to effectively treat the vectors being quantized as samples from a multivariate normal distribution. 
So for the analysis and design of PolarQuant we can assume without loss of generality that our goal is to quantize a random vector with multivariate Gaussian distribution.
A critical insight is that the distribution of angles after random preconditioning becomes predictable and can be analytically derived, which enables the design of optimal quantization schemes.

The polar distribution of a Gaussian vector is derived in the following lemma.

\begin{lemma}[Distribution of a Gaussian Vector Under Polar Transformation]\label{lem_polar_dist}
    For an integer power of two $d$, suppose that $\x \sim \Ncal(0, I_d)$ is a random zero mean isotropic Gaussian random variable in dimension $d$. 
    Let $\psi_{d}(\x) := \left( \psi^{(1)}, \psi^{(2)}, \ldots \psi^{(\log_2d)} \right)$ denote the set of polar angles obtained by applying the polar transformation defined in \cref{def_cartesian_to_polar} on $\x$. 
    Denote the radius of $\x$ by $r = \norm{\x}_2$.
    The joint probability density function for $\left( r, \psi^{(1)}, \psi^{(2)}, \ldots \psi^{(\log_2d)} \right)$ is the following:
    \begin{equation}
        f_{R, \Psi_d}(r, \psi_{d}(\x) ) = f_R(r) \cdot \prod_{\ell=1}^{\log_2d} f_{\Psi^{(\ell)}}\left( \psi^{(\ell)} \right), 
    \end{equation}
    where $f_R(r)$ is the p.d.f. defined in Fact~\ref{lem_norm_gaussian}, $f_{\Psi^{(1)}}$ is p.d.f. of the uniform distribution over $[0, 2\pi)^{d/2}$:
    \[
    f_{\Psi^{(1)}}: [0, 2\pi)^{d/2} \to (2\pi)^{-d/2},
    \]
    and for every $\ell\in\{ 2, 3, \ldots \log_2d \}$ the p.d.f. $f_{\Psi^{(\ell)}}$ is the following:
    \begin{align*}
        f_{\Psi^{(\ell)}}&: [0, \pi/2]^{d/2^\ell} \to \RR_+\\
        f_{\Psi^{(\ell)}}(\bpsi) &= \prod_{i=1}^{d/2^\ell} \frac{\Gamma(2^{\ell-1})}{2^{2^{\ell-1} - 2} \cdot \Gamma(2^{\ell-2})^2} \sin^{(2^{\ell-1}-1)}(2 \psi_i).
    \end{align*}
\end{lemma}

\begin{proof}
    The proof is by induction on $d$. First for the base of induction we prove the result in dimension $d=2$. So we prove that for a 2-dimensional random Gaussian vector $\y = (y_1, y_2) \in \RR^2$ if $(r, \theta)$ is the polar representation of this vector then following holds:
    \[
    f_{R, \Theta}(r, \theta ) = \frac{1}{2\pi} \cdot r \exp\left( -r^2/2 \right),
    \]
    To prove this, let $f_Y(\y)$ be the probability density function of the vector random variable $\y$. We know $\y$ has a normal distribution so we have:
    \begin{align*}
    f_{R, \Theta}(r, \theta ) &= r \cdot f_Y(\y) = r \cdot \frac{1}{2\pi} e^{-\frac{y_1^2 + y_2^2}{2}} = \frac{1}{2\pi} \cdot r e^{-\frac{r^2}{2}},
    \end{align*}
    where the first equality above follows from the change of variable from $(y_1, y_2)$ to $r = \sqrt{y_1^2 + y_2^2}$ and $\theta = \tan^{-1}(y_2 / y_1)$.
    This proves the base of induction for $d=2$.

    Now we prove the inductive step. 
    Suppose that the lemma holds for dimension $d/2$ and we want to prove it for dimension $d$. 
    % For ease of notation let 
    Denote
    $\theta := \psi^{(\log_2d)}$, $\bphi_1 := \left( \psi^{(l)}_{1: d/2^{l+1}} \right)_{\ell=1}^{\log_2d-1}$, $\bphi_2 := \left( \psi^{(l)}_{d/2^{\ell+1}+1: d/2^\ell} \right)_{\ell=1}^{\log_2d-1}$, $r_1 := \norm{\x_{1:d/2}}$, and $r_2 := \norm{\x_{d/2+1:d}}$.
    Essentially we sliced all the angle vectors $\psi^{(\ell)}$ in half and named the collection of first half vectors $\bphi_1$ and the collection of second halves $\bphi_2$.
    Using the definition of $\psi^{(\ell)}$'s in \cref{def_cartesian_to_polar}, $\bphi_1$ is exactly the polar transformation of $\x_{1:d/2}$, and $\bphi_2$ is the polar transformation of $\x_{d/2+1:d}$, so by the definition of $\psi_{d}(\x)$ in the lemma statement we have $\bphi_1 = \psi_{d/2}(\x_{1:d/2})$ and $\bphi_2 = \psi_{d/2}(\x_{d/2+1:d})$. 
    Thus, we can write:
    \begin{align}
    &f_{R, \Psi_d}(r, \psi_d(\x)) 
    = f_{R, \Theta, \Phi_1, \Phi_2}(r, \theta, \bphi_1, \bphi_2) \nonumber \\
    &= r \cdot f_{R_1, R_2, \Phi_1, \Phi_2}(r \cos \theta, r \sin \theta, \bphi_1, \bphi_2) \nonumber \\
    &= r \cdot f_{R_1, \Phi_1}(r \cos \theta, \bphi_1) \cdot f_{R_2, \Phi_2}( r \sin \theta, \bphi_2) \nonumber \\
    &= r \cdot f_{R, \Psi_{d/2}}(r_1, \bphi_1) \cdot f_{R, \Psi_{d/2}}( r_2, \bphi_2), \label{pdf_decompose_separable}
    \end{align}
    where the third line above follows from the change of variable from $(r_1, r_2) = (r\cos\theta, r\sin\theta)$ to $r = \sqrt{r_1^2 + r_2^2}$ and $\theta = \tan^{-1}(r_2 / r_1)$.
    In the fourth line above we used the definition of $\theta = \psi^{(\log_2d)} = \tan^{-1}\left( \frac{\norm{\x_{d/2+1:d}}_2 }{ \norm{\x_{1:d/2}}_2 } \right)$ from \cref{def_cartesian_to_polar}.
    
    Now if we let $f_{\Psi_{d/2}}(\bphi_1) := \prod_{\ell=1}^{\log_2d-1} f_{\Psi^{(\ell)}}\left( \psi^{(\ell)}_{1:d/2^{\ell+1}} \right)$ and $f_{\Psi_{d/2}}(\bphi_2) := \prod_{\ell=1}^{\log_2d-1} f_{\Psi^{(\ell)}}\left( \psi^{(\ell)}_{d/2^{\ell+1}+1:d/2^\ell} \right)$, by the inductive hypothesis we have  $f_{R, \Psi_{d/2}}(r_1, \bphi_1) = \frac{2}{2^{d/4} \cdot \Gamma(d/4)} r_1^{d/2-1} \exp\left( -r_1^2/2 \right) \cdot f_{\Psi_{d/2}}(\bphi_1)$ and $f_{R, \Psi_{d/2}}( r_2, \bphi_2) = \frac{2}{2^{d/4} \cdot \Gamma(d/4)} r_2^{d/2-1} \exp\left( -r_2^2 /2 \right) \cdot f_{\Psi_{d/2}}(\bphi_2)$.
    Plugging these values into \cref{pdf_decompose_separable} gives:
    \begin{align}
    f_{R, \Psi_{d}}(r, \psi_{d}(\x)) &= \frac{4\cdot (r_1 r_2)^{d-1}}{2^{d/2} \Gamma(d/4)^2} \exp\left( -r^2/2 \right) \cdot f_{\Psi_{d/2}}(\bphi_1) \cdot f_{\Psi_{d/2}}(\bphi_2)\nonumber \\
    &= \frac{2 r^{d-1} \cdot \sin^{d/2-1} (2\theta)}{2^{3d/2-2} \cdot \Gamma(d/4)^2} e^{ -r^2/2 }  \cdot f_{\Psi_{d/2}}(\bphi_1) \cdot f_{\Psi_{d/2}}(\bphi_2) \nonumber \\
    &= f_R(r) \cdot \frac{\Gamma(d/2)\cdot \sin^{d/2-1} (2\theta)}{2^{d/2-2} \cdot \Gamma(d/4)^2} f_{\Psi_{d/2}}(\bphi_1) \cdot f_{\Psi_{d/2}}(\bphi_2) \nonumber\\
    &= f_R(r) \cdot f_{\Psi_{d}}\left( \psi_{d}(\x) \right),
    \end{align}
    which completes the inductive proof of this lemma.
\end{proof}

\cref{lem_polar_dist} demonstrates that the angles of Gaussian vectors in polar coordinates have independent distributions, as the probability density function is separable. 
Moreover, all angles within the same level share identical distributions.
Specifically, at level $\ell$ all angles follow the distribution $\psi^{(\ell)}_i \sim \prod_{i=1}^{d/2^\ell} \frac{\Gamma(2^{\ell-1})}{2^{2^{\ell-1} - 2} \cdot \Gamma(2^{\ell-2})^2} \sin^{2^{\ell-1}-1}\left( 2 \psi^{(\ell)}_i \right)$.
This density becomes increasingly concentrated around $\pi/4$, particularly at higher levels $\ell$. 
This property is highly beneficial for reducing quantization error for the angles at higher levels.

\begin{algorithm}[t!]
\caption{PolarQuant}\label{alg:polar-quant}
\begin{algorithmic}[1]
    \STATE{\bfseries input:} embedding $\X \in \RR^{n \times d}$, %vectors $\x_1, \dots, \x_n \in \RR^{d}$, 
    precondition matrix $\S \in \RR^{d \times d}$, bit width $b$
    % \FOR{$i=1, \dots, n$}
    
    {\ttfamily\textcolor{blue}{// Cartesian to Polar transform}} 
    \STATE $\R_{i}, \bPsi^{(1)}_{i}, \dots, {\boldsymbol \Psi}^{(\log_2 d)}_{i} \gets $ \textsc{Polar}$(\X_{i} \cdot \S)$ for $i \in [n]$
    % \STATE $\r_i^{(0)} \leftarrow \x_i R$ for all $i \in [n]$
    % \ENDFOR
    % \FOR{$\ell=1, \dots, \log_2 d$}
    
    {\ttfamily\textcolor{blue}{// Codebook Construction}}
    \STATE Find partition intervals and centroids $( I^{(\ell)}_k, \theta^{(\ell)}_k)_{k \in [2^b]}$ of $\bPsi^{(\ell)} \in \RR^{n \times (d/2^\ell)}$ that minimize the cost in \cref{eq_quant_partition_centroid} for $\ell \in [\log_2 d]$ %using a numerical method 
    (See \cref{sec:practice} for details)\label{alg:codebook_construction}
    % \ENDFOR
    % \FOR{$i=1, \dots, n$ and $\ell=1, \dots, \log_2 d$}
    
    {\ttfamily\textcolor{blue}{// Angles Quantization}} \\
    \STATE
    $\J_i^{(\ell)} \gets \textsc{Quant}\left(\bPsi_i^{(\ell)}, (I^{(\ell)}_k, \theta^{(\ell)}_k)_{k \in [2^b]}\right)$ for $i \in [n]$ and $\ell \in [\log_2 d]$
    % \STATE $\left( \r_i, \psi^{(1)}, \psi^{(2)}, \ldots \psi^{(\log_2d)} \right) \gets$ \textsc{Polar}$(\S\cdot\x_i)$
    % \STATE $\q^{(1)}_i, \ldots \q^{(\log_2d)}_i \gets $ \textsc{Quant}($\bPsi_1^{(1)}, \ldots \bPsi_i^{(\log_2d)} $)
    % \ENDFOR
    \STATE {\bf output:} 
    % $\R \in \RR^{n \times 1}, \widehat{\bPsi}^{(1)} \in [b]^{n \times d/2}, \dots, \widehat{\bPsi}^{(\log_2 d)} \in [b]^{n \times 1}, ( I^{(\ell)}_k, \theta^{(\ell)}_k)_{k \in [2^b]}$
    $\R \in \RR^{n \times 1}, \J^{(1)}\in [2^b]^{n \times d/2}, \dots, \J^{(\log_2 d)} \in [2^b]^{n \times 1}, ( I^{(\ell)}_k, \theta^{(\ell)}_k)_{k \in [2^b]}$
    % $\left( \r_i, \q^{(1)}_i, \ldots \q^{(\log_2d)}_i \right)_{i \in [n]}$
   \\\hrulefill
   \STATE{\bf Procedure} \textsc{Polar} ($\y$) 
   % \hfill {\ttfamily\textcolor{blue}{//Polar transform}}
   % \STATE{\bf input:} embeddings $\x \in \mathbb{R}^{d}$, quant level $L \in [1, \log d]$, precondition matrix $R\in \mathbb{R}^{d \times d}$
   \STATE $\r^{(0)} \gets \y \in \RR^d$
   \FOR{$\ell = 1, \dots, \log_2d$}
   \FOR{$j = 1, \dots ,{d}/{2^{\ell}}$}
   \STATE $\bpsi^{(\ell)}_j \leftarrow \tan^{-1}\left(\r^{(\ell-1)}_{2j} / \r^{(\ell-1)}_{2j-1}\right)$
   \STATE $\r^{(\ell)}_j \leftarrow \norm{\r^{(\ell-1)}_{2j-1:2j}}_2$
   \ENDFOR{}
   \ENDFOR{}
   \STATE{\bf output:} $\r^{(\log_2d)}, {\bpsi}^{(1)}, \ldots, \bpsi^{(\log_2d)}$
   \\\hrulefill
   \STATE{\bf Procedure} \textsc{Quant} $\left(\bpsi, (I_k, \theta_k)_{k \in [2^b]}\right)$
   %\\ {\ttfamily\textcolor{blue}{//Angles Quantization}}
   \STATE $\boldsymbol{j}_i \gets \mathop{\text{argmin}}_{k \in [2^b]}  \abs{\theta_k - \bpsi_{i}}$ for $i \in [d']$ s.t. $\bpsi \in \RR^{d'}$
   % \STATE $\widehat{\bpsi}_i \gets \mathop{\text{argmin}}_{k \in [2^b]}  \abs{\theta_k - \bpsi_{i}} $ for $i \in [d']$ s.t. $\bpsi \in \RR^{d'}$
   % $i \in [\text{dim}(\bpsi)]$
    \STATE {\bf output:} $\boldsymbol{j}$
   % \STATE $\q^{(l)}_i \gets \left( i \in [2^b]:~ \psi^{(1)}_i \in I^{(l)}_i \right)$ for $i \in [2^b]$
   % \STATE {\bf output:} $\left( \q^{(1)}, \q^{(2)}, \ldots \q^{(\log_2d)} \right)$
   \\\hrulefill
   \STATE{\bf Procedure} \textsc{DeQuant}$(\r, (\boldsymbol{j}^{(\ell)})_{\ell \in [\log_2 d]}, (\theta_k^{(\ell)})_{k \in [2^b]}, \S)$
   %\\ {\ttfamily\textcolor{blue}{// Dequantize to Cartesian}}
   \FOR{$\ell=\log_2d, \dots, 1$}
   \FOR{$j=1, \dots, d/2^\ell $}
       \STATE $i \gets {\boldsymbol j}^{(\ell)}_j$ 
       \STATE $\r^{(\ell-1)}_{2j-1} \gets \r^{(\ell)}_{j} \cdot \cos \theta^{(\ell)}_{i}$ %for all $j \in [d/2^\ell]$
       \STATE $\r^{(\ell-1)}_{2j} \gets \r^{(\ell)}_{j} \cdot \sin \theta^{(\ell)}_{i}$ %for all $j \in [d/2^\ell]$
   \ENDFOR
   \ENDFOR{}
   \STATE {\bf output: $\r^{(0)} \cdot \S^\top$} 
\end{algorithmic}
\end{algorithm}

\subsection{PolarQuant Algorithm and Main Theorem} \label{sec:polar-quant}
PolarQuant starts by first applying random preconditioning, then transforming the vectors into polar coordinates, and finally quantizing each angle. 
Since \cref{lem_polar_dist} shows that the angles in polar coordinates are independent random variables, each angle can be quantized independently to minimize the total mean squared error.
Jointly quantizing multiple angle coordinates offers no additional benefit due to their independence, making our approach both computationally efficient and effective.
Therefore, we can focus on one angle at level $l$ and design optimal quantization scheme for it so as to minimize the mean squared error.

Consider an angle $\psi^{(\ell)}_i$ at some level $\ell$. 
According to \cref{lem_polar_dist}, its values lie within the range $[0, \pi/2]$ for $\ell\ge2$ and for $\ell=1$ it takes values in the range $[0, 2\pi)$ with a probability density function given by $f_\ell(\psi^{(\ell)}_i) := \frac{\Gamma(2^{\ell-1})}{2^{2^{\ell-1} - 2} \cdot \Gamma(2^{\ell-2})^2} \sin^{2^{\ell-1}-1}\left( 2 \psi^{(\ell)}_i \right)$.
The goal of quantization to $b$-bits is to partition the range $[0, \pi/2]$ (or $[0, 2\pi)$ in case of $\ell=1$) into $2^b$ intervals $I^{(\ell)}_1, I^{(\ell)}_2, \cdots I^{(\ell)}_{2^b}$ and find corresponding centroids $\theta^{(\ell)}_1, \theta^{(\ell)}_2, \ldots \theta^{(\ell)}_{2^b}$ such that the following is mean squared error is minimized:
\begin{equation}\label{eq_quant_partition_centroid}
\E_{\psi^{(\ell)}_i \sim f_\ell(\psi^{(\ell)}_i)}\left[ \sum_{j \in [2^b]:~ \psi^{(\ell)}_i \in I^{(\ell)}_j} \left| \psi^{(\ell)}_i - \theta^{(\ell)}_j \right|^2 \right].
\end{equation}
This problem is a continuous analog of the k-means clustering problem in dimension 1. 
Since we have an explicit formula for the p.d.f. of angle $\psi^{(\ell)}_i \sim f_\ell(\psi^{(\ell)}_i) = \frac{\Gamma(2^{\ell-1})}{2^{2^{\ell-1} - 2} \cdot \Gamma(2^{\ell-2})^2} \sin^{2^{\ell-1}-1}\left( 2 \psi^{(\ell)}_i \right)$ the optimal interval partitions and centroids for \cref{eq_quant_partition_centroid} can be efficiently computed using numerical methods. 
For example, one can run k-means clustering on the gathered angle values which can be considered samples from the distribution. 
This approach ensures minimal quantization error for each angle independently and the overall reconstruction error as well.

We provide a pseudocode of PolarQuant in \cref{alg:polar-quant}.
Our main result and error bound are proved in the following.
\begin{theorem}\label{main_thrm}
For a $d$-dimensional vector $\x \sim N(0, I_d)$, the polar quantization scheme in \cref{alg:polar-quant} uses $O(\log 1/\varepsilon)$ bits per coordinate + the space necessary to store $\|{\x}\|_2$, while reconstructing a vector $\x'$ from such a representation satisfying
\begin{align*}
    \E[\|\x - \x'\|_2^2] = \varepsilon \cdot \|\x\|_2^2.
\end{align*}
\end{theorem}

The proof of \cref{main_thrm} can be found in \cref{sec:proof_main_thrm}.
We note that a scheme which uses a deterministic $\varepsilon$-net $\mathcal{N}$ of the unit sphere $\mathbb{S}^{d-1}$, with $|\mathcal{N}| = O(1/\varepsilon)^d$ and rounds the vector $\hat{\x} = \x/\|\x\|_2$ also uses $O(\log 1/\varepsilon)$ bits per coordinate while achieving the above bounds in the worst case instead of in expectation over the Gaussian distribution. But our construction (i) gives the flexibility to vary the size of the codebook used per each level depending on the resource constraints and as the above theorem shows, can approach the same quality as pinning to an $\varepsilon$-net on average, (ii) does not need to store a $|\mathcal{N}|$-size codebook which is impractical even for modest sizes of $d$ and (iii) has a fast decoding/encoding implementation.  

\section{KV Cache Quantization with PolarQuant} \label{sec:polar-quant-kv}

In this section, we describe how PolarQuant can be applied to the KV cache problem and our practical implementation. Formally, given a stream of $(\q_1,\k_1,\v_1),\ldots,(\q_n,\k_n,\v_n)$, where $\q_i,\k_i,\v_i \in \RR^{d}$ are query, key and value embeddings at $i$-th generation step for all $i \in [n]$. Let $\K_{:i}, \V_{:i} \in \RR^{i \times d}$ be matrices defined by stacking $\k_1, \dots, \k_i$ and $\v_1, \dots, \v_j$ in their rows, respectively. The goal is to compute:
\begin{equation}\label{eq:attn-def}
    \text{softmax}\left(\frac{\K_{:i} \cdot \q_i}{\sqrt{d}}\right)^T\cdot \V_{:i}. 
\end{equation} 
% Keeping all of the key-value pairs in the cache is prohibitively expensive, especially for long sequences. Instead, we opt for approximate computation by sampling a few key-value pairs. Specifically, our goal is to construct an algorithm that at every time step $j$ computes an estimator $z_j$ for $\text{Attn}(q_j, K_j, V_j)$ in sublinear in $n$ time and memory. In particular for given precision $\varepsilon$, $z_j$ should satisfy the following error constraint:

For an efficient token generation, the KV cache at $i$-th generation step $(\K_{:i}, \V_{:i})$ are stored in the memory. To reduce the memory space, we invoke PolarQuant (\cref{alg:polar-quant}) on these embeddings. 
Let ${\whK}_{:i}, {\whV}_{:i} \in \RR^{i \times d}$ be their dequantizations using \textsc{DeQuant} procedure in \cref{alg:polar-quant}. Then, we estimate \cref{eq:attn-def} by computing
\begin{align} \label{eq:attn-approx}
\text{softmax}\left({\frac{{\whK}_{:i}  \cdot \q_i}{\sqrt{d}}}\right) \cdot {\whV}_{:i} .
\end{align}
Note that the na\"ive cache requires $d \cdot \bfpn$ memory space to store each $d$-dimensional embedding where $\bfpn$ is the number of bits to represent a single floating-point number. If we quantize $\log_2d$ level angles with $b$ bits each and keep centroids in $\bfpn$ bits, the memory space becomes $\left( \bfpn + (d-1)b\right)$. 
% In other words, when $b=\alpha \cdot \bfpn$ for some $\alpha \in (0,1)$ the compression ratio is $\alpha + \frac1d$. 
For example, \llama{} is represented by $\bfpn=16$ bits and has $d=128$. For $b=3$, we can save the memory space $4.008$ times. In \cref{sec:exp}, the PolarQuant with KV cache marginally degrades the performance of LLMs on various tasks.

\begin{figure*}[t]
    \vspace{-0.1in}
    \centering
    \begin{subfigure}[t]{0.48\textwidth}
        \includegraphics[width=\textwidth]{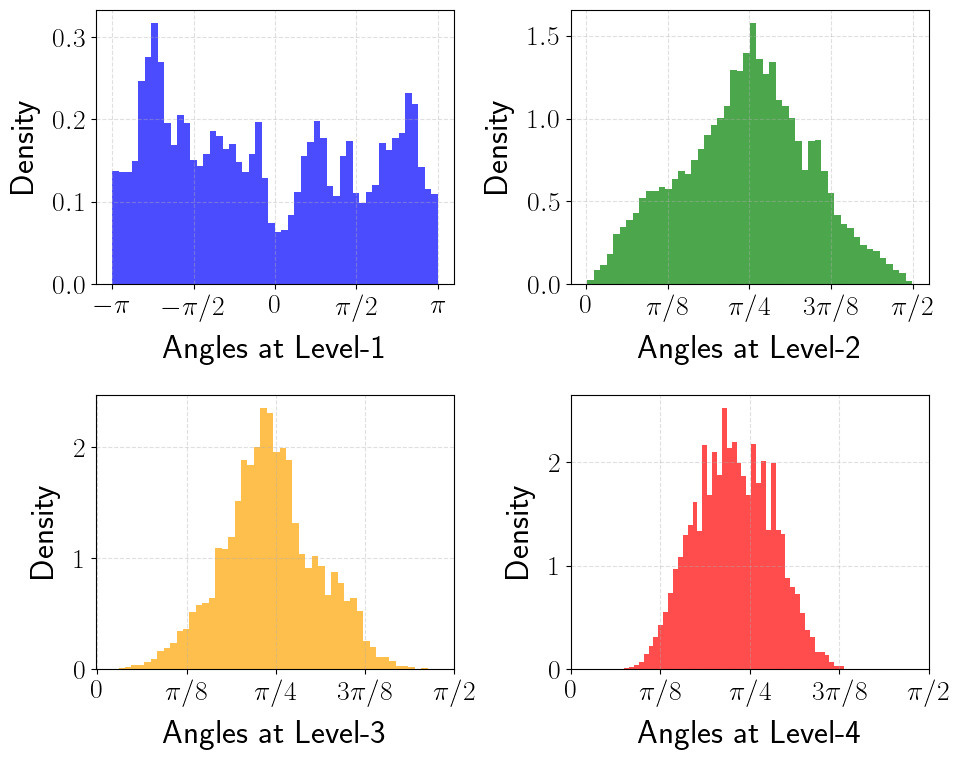}
        \caption{without random preconditioning}
    \end{subfigure}
    \begin{subfigure}[t]{0.48\textwidth}
        \includegraphics[width=\textwidth]{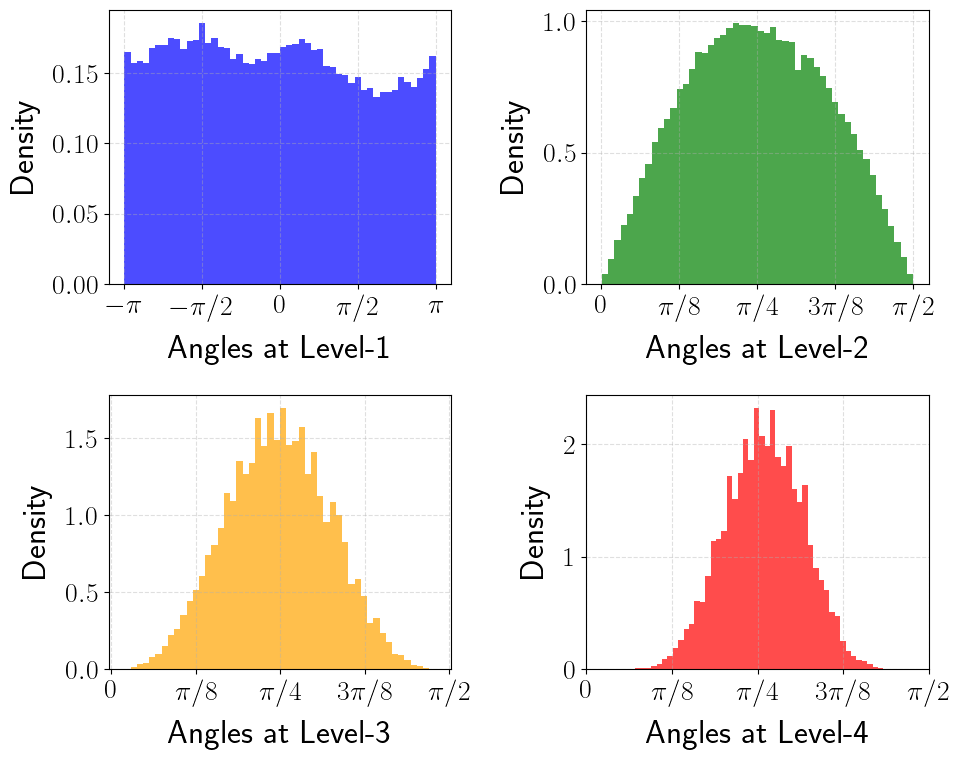}
        \caption{with random preconditioning}
    \end{subfigure}
    \vspace{-0.05in}
    \caption{Distributions of angles of polar transformed key embeddings (a) with and (b) without random preconditioning. Preconditioning flattens the angle distribution and removes outliers which allows angle quantization more accurately.} \label{fig:angle-distribution}
    \vspace{-0.1in}
\end{figure*}

\subsection{Practical Implementation} \label{sec:practice}

The PolarQuant algorithm recursively reduces the dimension of radii by half until the input has dimension 1. 
We recurse on the polar transformation for a constant $L=4$ levels. 
Thus, for an embedding of dimension $d$, we obtain $d/16$-dimensional radii and $15d/16$ angle values. 
We also define different numbers of bits for each quantization level: $b=4$ bits for the first level, and $b=2$ bits for the remaining levels. This is because the range of angle at the first level $[0, 2\pi)$ is $4$ times wider than the others $[0, \pi/2]$. Consequently, the representation of a block of 16 coordinates uses $b_{\text{FPN}} + 32 + 8 + 4 + 2 = b_{\text{FPN}} + 46$ bits that translates to $62 / 16 = 3.875$ bits per coordinate when $b_{\text{FPN}} = 16$ bits.

We implement PolarQuant using the Pytorch~\cite{paszke2019pytorch} framework. Since the smallest data type is represented in 8 bits ~($\mathtt{torch.uint8}$), we pack quantized angle indices into 8-bit unit. To accelerate computation on GPU clusters, we implement CUDA kernels for two key operations: (1) the product of query vectors with the dequantized key cache, i.e., ${\whK}_{:i}  \cdot \q_i$, and (2) the product of attention scores with the dequantized value cache as per \cref{eq:attn-approx}. 
For the preconditioning matrix $\S$, we generate a random rotational matrix. The matrix $\S$ is shared across key and value embeddings, as well as all layers and attention heads in the Transformer architecture.

For angle codebook construction (line 3 in \cref{alg:polar-quant}), we use the 1-D k-means++ clustering on either online angles obtained from polar-transformed inputs or offline precomputed angles. Both approaches approximate the solution to \cref{eq_quant_partition_centroid} by discretizing with samples from angle distributions. 
While online approach requires additional clustering computation during every prefill stage, this one-time cost is offset by improved performance compared to the offline approach. We present detailed runtime and performance comparisons in \cref{sec:exp}.

% For example, when angle indices are represented by $4(=\log_2 16)$ bits (as in the first level), then we merge every pair of consecutive indices into a single 8-bit value by computing $16 \cdot I_1 + I_2$, where $I_1, I_2 \in [16]$. 

\begin{figure*}[t]
    \centering
    \vspace{-0.12in}
    \begin{subfigure}[t]{0.3\textwidth}
        \centering
        \includegraphics[width=\textwidth]{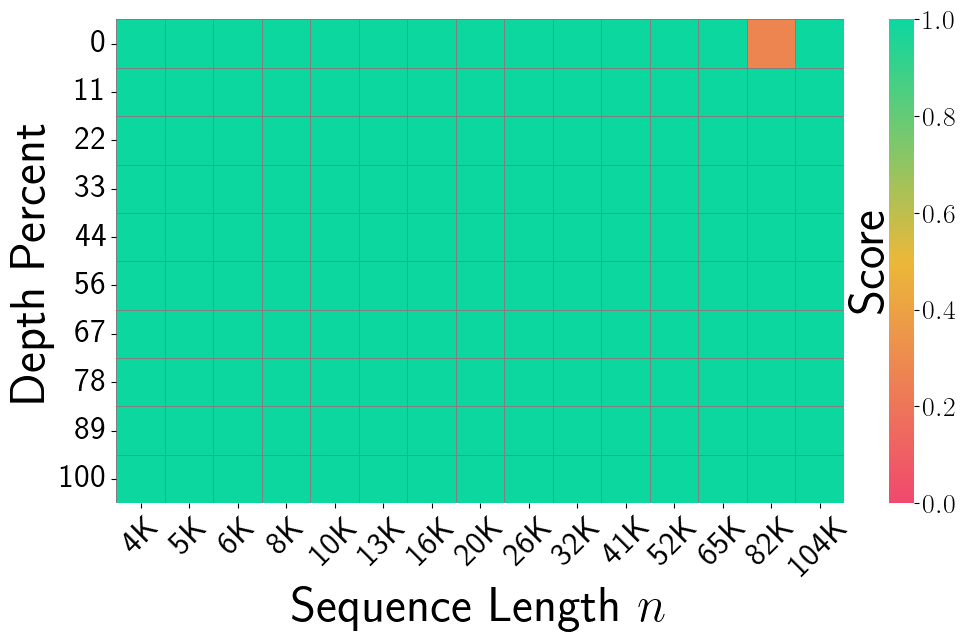}
        \caption{Exact (16 bits), Score: $0.995$}
        \vspace{-0.02in}
    \end{subfigure}
    \begin{subfigure}[t]{0.3\textwidth}
        \centering
        \includegraphics[width=\textwidth]{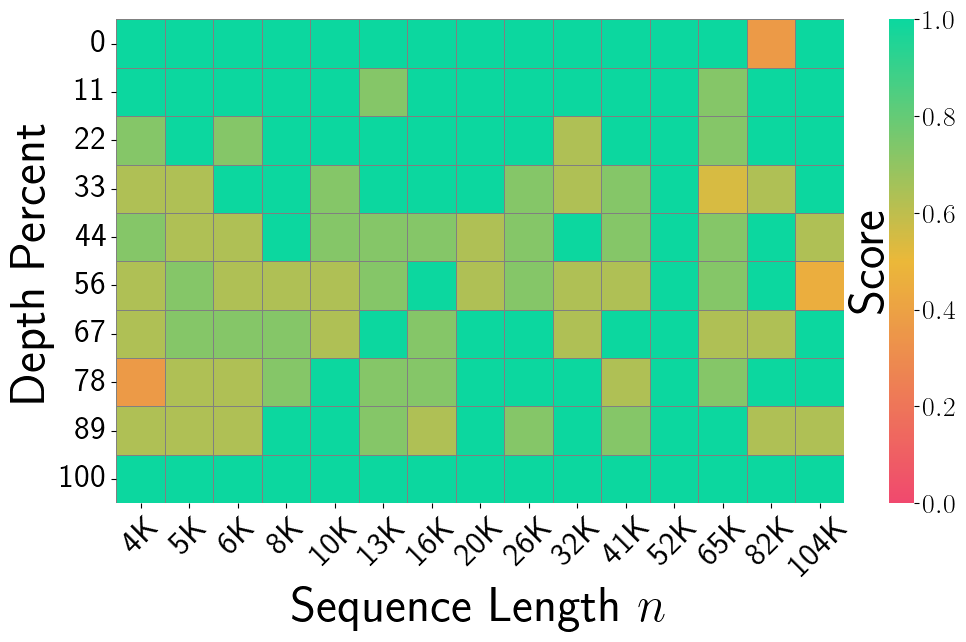}
        \caption{SnapKV, Score: $0.858$}
        \vspace{-0.02in}
    \end{subfigure}
    \begin{subfigure}[t]{0.3\textwidth}
        \centering
        \includegraphics[width=\textwidth]{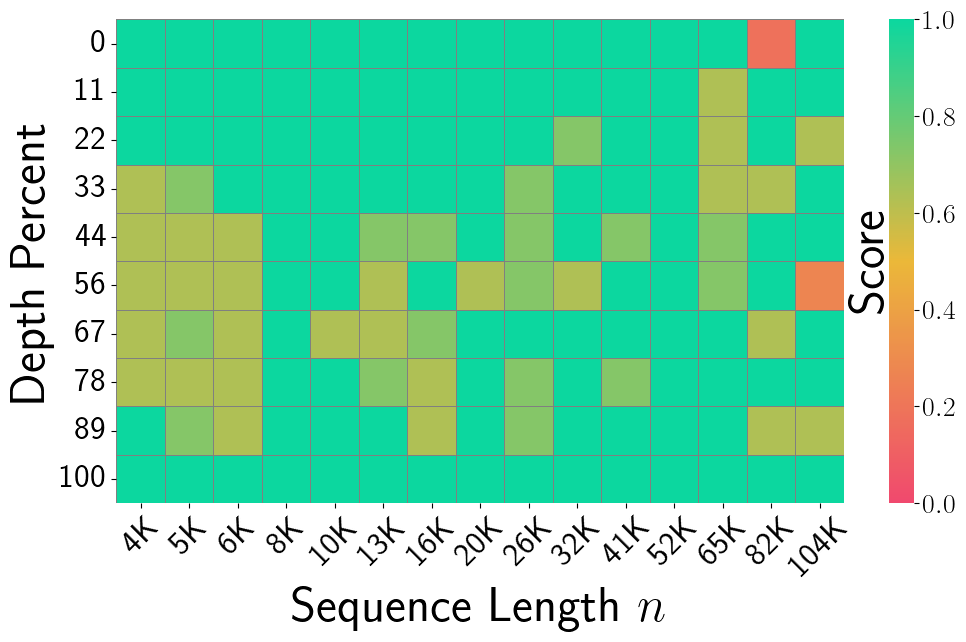}
        \caption{PyramidKV, Score: $0.891$}
        \vspace{-0.02in}
    \end{subfigure}
    \begin{subfigure}[t]{0.3\textwidth}
        \centering
        \includegraphics[width=\textwidth]{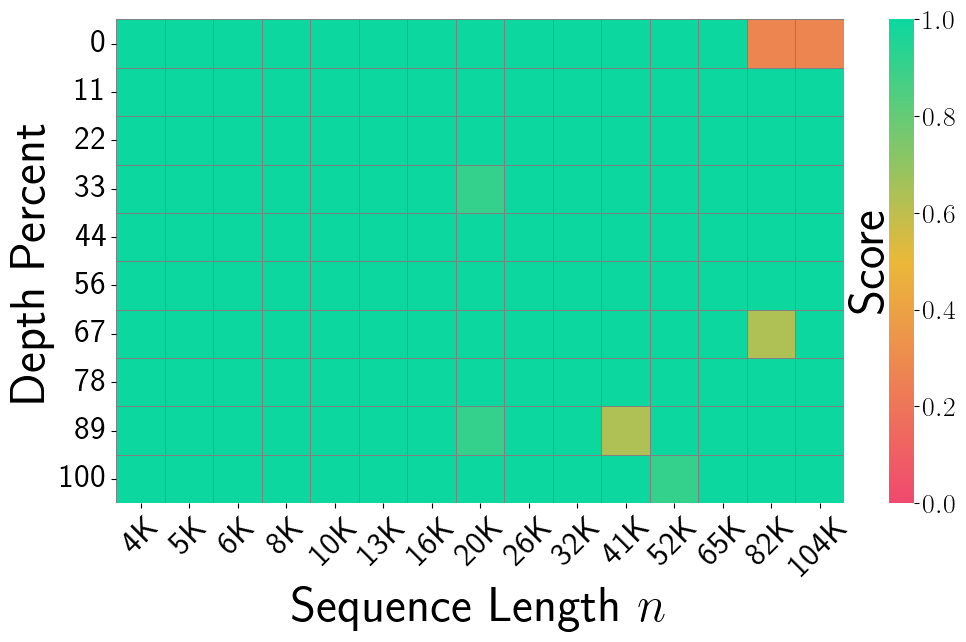}
        \caption{KIVI, Score: $0.984$}
        % \caption{Subplot 1}
    \end{subfigure}
    \begin{subfigure}[t]{0.3\textwidth}
        \centering
        \includegraphics[width=\textwidth]{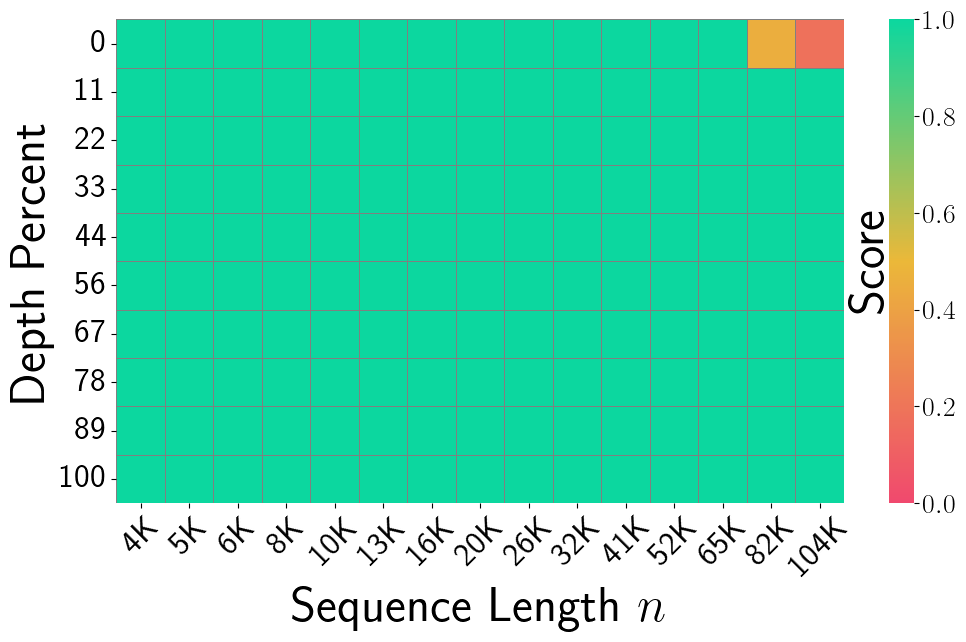}
        \caption{PolarQuant, Score: $0.991$}
        % \caption{Subplot 1}
    \end{subfigure}
    \begin{subfigure}[t]{0.3\textwidth}
        \centering
        \includegraphics[width=\textwidth]{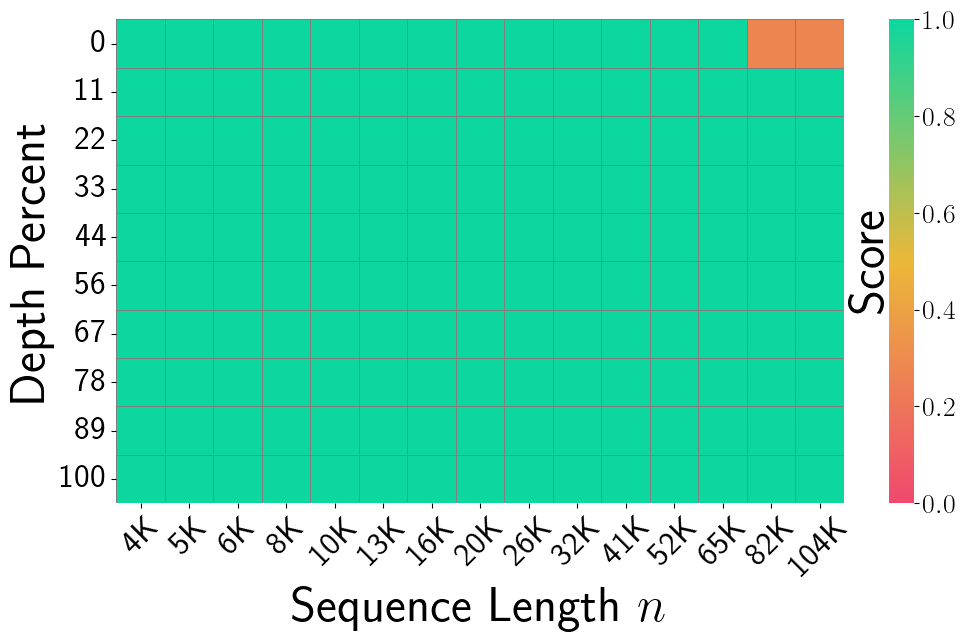}
        \caption{PolarQuant-R, Score: $0.990$}
    \end{subfigure}
    \vspace{-0.1in}
    \caption{Needle-In-A-Haystack test using \llama{}. The test spans different depths and context lengths ranging from 4K to 104K. Green/red colors indicate high/low recall scores (higher is better). PolarQuant shows the best performance.} \label{fig:niah}
    \vspace{-0.1in}
\end{figure*}

\section{Experiments}\label{sec:exp}All experiments are performed with a single NVIDIA RTX A6000 GPU with 48GB VRAM.

\vspace{-0.1in}
\subsection{Random Precondition on KV Cache}

We first explore the effectiveness of preconditioning. In particular, we choose a single prompt from Qasper dataset in LongBench~\cite{bai2023longbench} and extract the corresponding KV cache. To observe how preconditioning improves, we transform the KV cache into 4-level polar coordinates and plot their angle distributions of the key cache. Note that the first level angles are range in $[0, 2\pi)$ and the rest are in $[0, \pi/2]$. The results are illustrated in \cref{fig:angle-distribution}. As shown in \cref{lem_polar_dist}, the distribution of angles get predictably sharper around $\pi/4$ as the level increases. Moreover, we observe that at the first level the preconditioning flattens the angle distribution and removes outliers. This allows us to quantize angles in the KV cache more accurately.

\subsection{Needle-In-A-Haystack} \label{sec:niah}

Next we evaluate our method for the ``Needle-In-A-Haystack" test~\cite{niah}. It asks the model to retrieve the information in a given sentence where the sentence (the ``needle'') is placed in an arbitrary location of a long document (the ``haystack''). We follow the same setting from~\citet{fu2024data} and use the \llama{} to run the test. 
We vary the input sequence lengths from 4K to 104K. 
The evaluation is based on the recall score by comparing the hidden sentence. We compare PolarQuant to SnapKV~\cite{li2024snapkv}, PyramidKV~\cite{cai2024pyramidkv} and KIVI~\cite{liu2024kivi}, where we use their implementations from~\cite{jiang2024minference}. 
% We set hyperparameters to get a memory compression ratio of $0.25$. 
All methods are set to a compression ratio of $0.25$, i.e., required memory is $\times 0.25$ the full KV cache.  
Specifically, we run our algorithm with and without the preconditioning and refer to them as PolarQuant-R and PolarQuant, respectively. In \cref{fig:niah}, we observe that quantization methods (e.g., KIVI, PolarQuant) outperform token-level compression methods (e.g., SnapKV, PyramidKV). PolarQuant shows better scores than KIVI. Additionally, PolarQuant shows a marginally better score than PolarQuant-R. 

\subsection{End-to-end Generation on LongBench}

We run various KV cache compression algorithms for LongBench datasets~\cite{bai2023longbench}, which encompasses diverse long-text scenarios including single/multi-document question-answering (SQA/MQA), summarization (Sum), few-shot learning (Few), synthetic tasks (Syn), and code completion (Code).
Since the number of generated tokens is small compared to the input sequence length across all datasets, we preserve all new streamed query, key, and value pairs from the generation stage in full precision (16 bits) for all methods. We evaluate PolarQuant against the baseline methods using in \cref{sec:niah} as well as StreamingLLM~\cite{xiao2023efficient} and HeadKV~\cite{fu2024not} on \llama{}. 

We investigate two variants of PolarQuant-R: one using online codebook construction and another using offline one discussed in \cref{sec:practice}. The online variant performs clustering for each individual input prompt and layer, while the offline one employs a single precomputed codebook that is shared across all input prompts, layers, and attention heads. This offline approach is supported by our findings that the angle distribution, when preconditioned, remains consistent regardless of the input.

% Following~\cite{bai2023longbench}, we use task-specific metrics such as F1 scores for question-answering tasks. 

As reported in \cref{tab:performance_comparison}, our methods achieve superior performance compared to other methods, i.e., the average performance scores are higher by a large margin. This justifies the performance benefits of the quantization of polar coordinates. Moreover, the preconditioned variants (PolarQuant-R) generally demonstrates better performance than the non-preconditioned version. Among them, the online variant performs slightly better than the offline one. 
% However, for summarization tasks, we observe that the basic PolarQuant achieves higher scores than PolarQuant-R. This can be attributed to the nature of summarization tasks, where certain key sentences need to be preserved in the KV cache. The preconditioning tends to flatten the embeddings, which can smooth out important information and consequently impact the end-to-end performance of this task.

\setlength{\tabcolsep}{12pt}
\def\arraystretch{1.1}%
\begin{table}[t]
\centering
\caption{LongBench-V1~\cite{bai2023longbench} results of various KV cache compression methods on \llama{}. The best values among compression methods are indicated in {\bf bold}.}
\vspace{0.05in}
\scalebox{0.9}{
\begin{tabular}{lccccccc}
\toprule
% Method  & Single-Doc QA & Multi-Doc QA & Summarization & Few-shot & Synthetic & Code & Average \\
% Method & SQA & MQA & SUM & FEW & SYN & CODE & Average \\
\multirow{ 2}{*}{Method} & \multicolumn{6}{c}{Task} & \multirow{ 2}{*}{Average}\\
  & SQA & MQA & Sum & Few & Syn & Code &  \\
\midrule
Exact (16 bits)  & 45.71 & 45.32 & 26.69 & 68.62 & 59.25 & 46.17 & 48.63 \\
\midrule
Snapkv & 38.23 & 42.61 & 19.07 & 64.65 & 59.60 & 43.28 & 44.57 \\
HeadKV & 39.45 & 42.69 & 19.77 & 68.07 & 59.48 & 42.60 & 45.34 \\
PyramidKV & 36.80 & 41.54 & 18.91 & 64.88 & {59.68} & 42.38 & 44.03 \\
StreamingLLM & 25.68 & 35.79 & 20.90 & 56.91 & 58.81 & 32.07 & 38.36 \\
KIVI & 43.38 & 37.81 & {\bf 27.44} & 68.60 & 58.67 & 44.29 & 46.70 \\
PolarQuant & 44.03 &{44.34} & 27.32 & {\bf 68.68} & {\bf 59.82} & 44.46 & 48.11 \\
\makecell[tl]{PolarQuant-R (offline)} & 44.71 & 44.72 & 26.43 & 68.58 & 60.08 & {\bf 45.20} & 48.29 \\
% PolarQuant with RP on K& {45.31} & 44.02 & 25.95 & {\bf 68.74} & {\bf 59.90} & {\bf 45.37} & { 48.22} \\
% \multirow{}PolarQuant-R & 44.71 & 44.72 & 26.43 & 68.58 & 60.08 & 45.20 & 48.29 \\
\makecell[tl]{PolarQuant-R (online)} & {\bf 45.45} & {\bf 45.13} & 26.42 & 68.54 & 59.57 & {45.13} & {\bf 48.37} \\
\bottomrule
\end{tabular}
% \begin{tabular}{lcccccccccccccc}
% \toprule
% \textbf{Method}       & \textbf{qasper} & \textbf{multa\_en} & \textbf{hotpotqa} & \textbf{2wikimqa} & \textbf{gov\_port} & \textbf{multnews} & \textbf{trec} & \textbf{triviaqa} & \textbf{samsum} & \textbf{passount} & \textbf{passl\_en} & \textbf{lcc} & \textbf{repoch-p} & \textbf{average} \\
% \midrule
% Exact (16 bits)          & 0.4287 & 0.4854 & 0.5205 & 0.3860 & 0.3131 & 0.2207 & 0.7167 & 0.9185 & 0.4236 & 0.2037 & 0.9813 & 0.4962 & 0.4273 & 0.5017 \\
% \midrule
% StreamingLLM   & 0.2065 & 0.3071 & 0.3914 & 0.3243 & 0.2310 & 0.1870 & 0.5800 & 0.8387 & 0.2885 & 0.2036 & 0.9726 & 0.3369 & 0.3046 & 0.3979 \\
% SnapKV         & 0.3391 & 0.4255 & 0.4909 & 0.3613 & 0.2048 & 0.1767 & 0.6200 & 0.9170 & 0.4023 & 0.2033 & 0.9886 & 0.4670 & 0.3986 & 0.4612 \\
% PyramidKV      & 0.3386 & 0.3975 & 0.4712 & 0.3596 & 0.2003 & 0.1778 & 0.6367 & 0.9076 & 0.4021 & 0.2040 & 0.9897 & 0.4525 & 0.3951 & 0.4564 \\
% HeadKV         & 0.3623 & 0.4266 & 0.4834 & 0.3704 & 0.2067 & 0.1887 & 0.7133 & 0.9191 & 0.4098 & 0.2043 & 0.9853 & 0.4735 & 0.3784 & 0.4709 \\
% KIVI           & 0.4202 & 0.4473 & 0.4262 & 0.3300 & 0.3274 & 0.2214 & 0.7200 & 0.9200 & 0.4180 & 0.2045 & 0.9688 & 0.4770 & 0.4087 & 0.4838 \\
% \midrule
% PolarQuant (our) & 0.4172 & 0.4634 & 0.5045 & 0.3823 & 0.3250 & 0.2215 & 0.7200 & 0.9166 & 0.4238 & 0.2033 & 0.9930 & 0.4764 & 0.4129 & {\bf 0.4969} \\
% \bottomrule
% \end{tabular}
}
\label{tab:performance_comparison}
\end{table}

\begin{table}[h!]
\centering
\caption{Wall-clock runtime comparisons of various KV cache compression methods. The input sequence length is $n=16{,}384$ and the number of generated tokens is $1{,}024$.}
\vspace{0.05in}
\scalebox{0.9}{
\begin{tabular}{lcc}
\toprule 
Method & Prefill Time (sec) & Generation Time (sec) \\
\midrule
Exact (16 bits) & 2.934 & 38.374 \\
SnapKV &  3.438 & 34.053 \\
PyramidKV & 3.428 & 32.732 \\
HeadKV & 3.300 & 34.401 \\
KIVI & 3.590 & 49.564 \\ 
\midrule
PolarQuant & 11.623 & 43.652 \\
\makecell[tl]{PolarQuant-R (online)} & 11.633 & 44.448\\
\makecell[tl]{PolarQuant-R (offline)} & 3.364 & 44.097\\
\bottomrule
\end{tabular} \label{tab:runtime}
}
\end{table}

\iffalse
\begin{table*}[t]
\centering
\caption{Performance comparison of various KV cache compression methods on \mistral including StreamingLLM~\cite{xiao2023efficient}, SnapKV, PyramidKV, HeadKV, KIVI and our method across LongBench-V1~\cite{bai2023longbench} datasets.}
\vspace{0.1in}
\scalebox{0.9}{
\begin{tabular}{lccccccc}
\toprule
  & Single-Doc QA & Multi-Doc QA & Summarization & Few-shot & Synthetic & Code & Average \\
\midrule
\multicolumn{8}{c}{\llama} \\
\midrule
Exact (16 bits) & 51.43 & 56.99 & 26.12 & 69.15 & 58.00 & 54.52 & 52.70 \\
SnapKV & 44.71 & 55.65 & 18.76 & 65.62 & 58.00 & 51.03 & 48.96 \\
PyramidKV & 43.48 & 55.32 & 18.45 & 65.06 & 57.83 & 49.13 & 48.21 \\
StreamingLLM & 33.88 & 48.05 & 20.50 & 67.21 & 58.50 & 50.37 & 46.42 \\
KIVI & 50.50 & 56.29 & 26.06 & 69.04 & 58.17 & 53.04 & 52.18 \\
\midrule 
PolarQuant w/o RP & 50.96 & 57.09 & 26.57 & 69.04 & 58.17 & 52.13 & 52.33 \\
PolarQuant with RP on K & 50.90 & 56.96 & 26.25 & 68.84 & 57.67 & 54.05 & 52.44 \\
PolarQuant with RP on KV & 50.76 & 57.15 & 25.70 & 68.97 & 57.83 & 53.92 & 52.39 \\
\bottomrule
\end{tabular}
}
\end{table*}
\fi

\vspace{-0.1in}
\subsection{Runtime Analysis}
\vspace{-0.02in}

We evaluate wall-clock runtimes of both prefill and token generation stages. Using the Llama model with an input prompt length of $16{,}384$, we measure the time to generate $1{,}024$ tokens for each method. Table~\ref{tab:runtime} summaries the result. Token eviction approaches (SnapKV, PyramidKV, and HeadKV) demonstrate faster generation times compared to exact and quantization methods, though at the cost of lower quality. Among quantization approaches, our PolarQuant algorithms achieve 14\% faster generation time than the KIVI while maintaining superior performance. These results demonstrate that PolarQuant offers advantages in both computational efficiency and model performance. 
% For faster prefill time, we suggest to use offline codebook construction as its runtime significantly decreases without k-means++ clustering, though sacrificing subtle performance. We remain a better codebook construction method for future work. 
To achieve faster prefill times, we recommend using offline codebook construction, as it significantly reduces runtime by eliminating the need for clustering, though this results in a modest performance trade-off. 
We leave even better codebook construction approaches for future research. 
% In addition, the preconditioning adds negligible overhead. 

\section{Conclusion}
We propose PolarQuant, a novel quantization method applied to angles in polar coordinates. We connect it to the random preconditioning which allows us to formalize angle distribution to be quantized. We provide rigorous theoretical bounds on quantization error. When applied to the KV cache compression problem, PolarQuant significantly reduces memory requirements during LLM inference while maintaining model performance. The principles underlying our method extend beyond KV cache compression, offering potential applications in LLM weight quantization and general vector similarity search problems.

\bibliography{paper}
\bibliographystyle{icml2025}

\newpage
\appendix
\onecolumn

\section{Proof of Fact~\ref{lem_norm_gaussian}} \label{sec:proof_lem_norm_gaussian}

\begin{proof}
    The cumulative distribution function (c.d.f) of the random variable $R$ can be computed as follows:
    \begin{align*}
    F_R(r) &:= \Pr_{\x}[\norm{\x}_2 \le r]\\ 
    &= \Pr_{\x}[\norm{\x}_2^2 \le r^2] = \frac{1}{\Gamma(d/2)} \gamma(d/2, r^2/2), 
    \end{align*}
    where the last equality is because the squared norm of $\x$, by definition, is Chi-squared random variable.
    Differentiating the above c.d.f gives us the p.d.f of $R$:
    \[
    f_R(r) = \frac{2}{2^{d/2} \cdot \Gamma(d/2)} r^{d-1} \exp\left( -r^2/2 \right).
    \qedhere\]
\end{proof}

\section{Error Bounds for Polar Quantization}\label{sec:lem_quant_error}
\begin{lemma}\label{lem_var_theta}
    If $X$ and $Y$ are independent non-negative random variables that are sampled from the generalized gamma distribution with probability density function $f_Z(z) = \frac{2}{2^{d/2} \cdot \Gamma(d/2)}z^{d-1}\exp(-z^2/2)$ so that $X^2, Y^2 \sim \chi^2_d$, then the distribution of $\Theta = \tan^{-1}(Y/X)$ has the pdf
    \begin{align*}
        f_{\Theta}(\theta) = \frac{\Gamma(d)}{2^{d-2}\Gamma(d/2)^2}\cdot \sin^{d-1}(2\theta), \quad 0 \le \theta \le \pi / 2.
    \end{align*}
    We also have that $\mu_{\Theta} = \E[\Theta] = \pi/4$ and $\sigma_{\Theta} = \sqrt{\textnormal{Var}(\Theta)} = O(1/\sqrt{d})$.
\end{lemma}
\begin{proof}
    Since $X$ and $Y$ are i.i.d., their joint distribution function is the following:
    \begin{align*}
    f_{X,Y}(x, y) = f_Z(x) \cdot f_Z(y) 
    = \left( \frac{2}{2^{d/2} \cdot \Gamma(d/2)} \right)^2 (x \cdot y)^{d-1} \exp\left( -(x^2 + y^2)/2 \right).
    \end{align*}
    Now by changing the coordinates from Cartesian to polar we can represent the above distribution as a joint distribution over variables $r = \sqrt{x^2+y^2}, \theta = \tan^{-1}(y/x)$ with $r \ge 0$ and $\theta \in [0, \pi/2)$ as follows:
    \begin{align*}
    f_{R, \Theta}(r, \theta) = r \cdot f_{X, Y}\left( r\cos \theta, r \sin \theta \right) = \frac{r^{2d-1}}{2^{2d - 3} \cdot \Gamma(d/2)^2} \exp\left( -r^2/2 \right) \cdot \sin^{d-1}(2\theta).
    \end{align*}
    Since the joint probability distribution of $r, \theta$ is a separable function, we can deduce the marginal probability distribution of $\theta$ as follows:
    \begin{align*}
        f_{\Theta}(\theta) &= \int_0^\infty f_{R, \Theta}(r, \theta) dr \\
        &= \frac{ \sin^{d-1}(2\theta) }{2^{2d - 3} \cdot \Gamma(d/2)^2} \cdot \int_0^\infty r^{2d-1} \exp\left( -r^2/2 \right)  dr \\
        &= \frac{ \sqrt{2\pi} \cdot \sin^{d-1}(2\theta) }{2^{2d - 2} \cdot \Gamma(d/2)^2} \cdot \int_{-\infty}^\infty \frac{|r|^{2d-1}}{\sqrt{2\pi}} e^{ -r^2/2 }  dr \\
        &= \frac{ \sqrt{2\pi} \cdot \sin^{d-1}(2\theta) }{2^{2d - 2} \cdot \Gamma(d/2)^2} \cdot \E_{r \sim \Ncal(0,1)} \left[ |r|^{2d-1} \right] \\
        &= \frac{\Gamma(d)}{2^{d - 2} \cdot \Gamma(d/2)^2} \cdot \sin^{d-1}(2\theta),
    \end{align*}
    where the last equality above follows from Fact~\ref{moment_normal}. For a proof of the mean and variance bound see \cref{lem_var_theta}.

    From the symmetry of $f_{\Theta}(\cdot)$ around $\theta = \pi / 4$, it is clear that $\mu_{\Theta} = \pi / 4$. We now bound $\sigma_{\Theta}$.
    \begin{align*}
        \text{Var}(\Theta) &= \frac{\Gamma(d)}{2^{d-2}\Gamma(d/2)^2}\int_{0}^{\pi/2}(\theta - \pi / 4)^2 \sin^{d-1}(2\theta)\, \mathrm{d}\theta\\
        &= \frac{\Gamma(d)}{2^{d-2}\Gamma(d/2)^2} \int_{-\pi/4}^{\pi/4} \theta^2 \sin^{d-1}(2\theta + \pi / 2)\, \mathrm{d}\theta\\
        &= \frac{2\Gamma(d)}{2^{d-2}\Gamma(d/2)^2}\int_{0}^{\pi/4}\theta^2 \cos^{d-1}(2\theta)\, \mathrm{d}\theta\\
        &= \frac{\Gamma(d)}{2^{d-3}\Gamma(d/2)^2}\int_{0}^{\pi/4}\theta^2 (1 - 2\sin^2 \theta)^{d-1} \mathrm{d}\theta\\
        &\le \frac{\Gamma(d)}{2^{d-3}\Gamma(d/2)^2}\int_{0}^{\pi/4}\theta^2 (1 - \frac{8\theta^2}{\pi^2})^{d-1} \mathrm{d}\theta \qquad \text{(since $\sin(x) \ge 2x/\pi$ for $0 \le x \le \pi/2)$}\\
        &\le \frac{\Gamma(d)}{2^{d-3}\Gamma(d/2)^2}\int_{0}^{\pi/4} \theta^2 \exp\left(-\frac{8(d-1)\theta^2}{\pi^2}\right)\, \mathrm{d}\theta \qquad \text{(since $1-x \le \exp(-x)$)}\\
        &\le \frac{\Gamma(d)}{2^{d-3}\Gamma(d/2)^2}\int_{0}^{\infty}\theta^2\exp\left(-\frac{8(d-1)\theta^2}{\pi^2}\right)\, \mathrm{d}\theta \ .
    \end{align*}
    Substituting $\beta = 8(d-1)\theta^2/\pi^2$, we have
    \begin{align*}
        \text{Var}(\Theta) &\le \frac{\Gamma(d)}{2^{d-3}\Gamma(d/2)^2} \cdot \frac{\pi^3}{32\sqrt{2}(d-1)^{3/2}}\int_{0}^{\infty}\beta^{1/2}\exp(-\beta)\, \mathrm{d}\beta \le C \frac{\Gamma(d)}{2^{d-3}\Gamma(d/2)^2(d-1)^{3/2}}.
    \end{align*}
    Assuming $d$ is even and using Stirling approximation, we get
    \begin{align*}
        \Gamma(d) = \frac{d!}{d} \approx \frac{\sqrt{2\pi d} \left({d}/{e}\right)^d}{d}\, \text{and}\, \Gamma(d/2) = \frac{(d/2)!}{(d/2)} \approx \frac{\sqrt{2\pi(d/2)}(d/2e)^{d/2}}{d/2}.
    \end{align*}
    Therefore, 
    \begin{align*}
        \text{Var}(\Theta) \le C' \frac{2^d \sqrt{d}}{2^{d-3}(d-1)^{3/2}} \le \frac{C''}{d-1}
    \end{align*}
    where $C''$ is a universal constant independent of $d$ hence proving the theorem.
\end{proof}
\section{Proof of Theorem~\ref{main_thrm}} \label{sec:proof_main_thrm}
Suppose that $X, Y$ are random variables as in Lemma~\ref{lem_var_theta} so that $X^2, Y^2 \sim \chi^2_d$ and define $R = \sqrt{X^2 + Y^2}$ and $\Theta = \tan^{-1}(Y/X)$. Let $\{\theta_1, \ldots, \theta_k\} \subseteq [0, \pi/2]$ be a codebook and define $\Theta' = \text{round}(\Theta)$ to be the $\theta_i$ nearest to $\Theta$ and $R'$ be an approximation of $R$. We then define
\begin{align*}
    X' = R' \cdot \cos(\text{round}(\Theta))\qquad \text{and}\qquad Y' = R' \cdot \sin(\text{round}(\Theta))
\end{align*}
to be the reconstructions of $X$ and $Y$ respectively from the rounding scheme, which rounds $\Theta$ using the codebook and the radius $R$ using a recursive approximation. The reconstruction error is defined as
\begin{align*}
    & (X - X')^2 + (Y - Y')^2\\
    &= (R \cos(\Theta) - R' \cos(\Theta'))^2 + (R  \sin(\Theta) - R' \sin(\Theta'))^2\\
    &= (R \cos(\Theta) - R \cos(\Theta') + R\cos(\Theta') - R'\cos(\Theta'))^2 + (R\sin(\Theta) - R\sin(\Theta') + R\sin(\Theta') -R'\sin(\Theta'))^2.
\end{align*}
Using the fact that $2ab \le (1/\alpha) \cdot a^2 + \alpha \cdot b^2$ for any $\alpha > 0$, we get $(a+b)^2 = a^2 + 2ab + b^2 \le (1 + 1/\alpha)a^2 + (1+\alpha)b^2$ for any $\alpha > 0$. Therefore we have that for any $\alpha > 0$,    
\begin{align*}
   (X - X')^2 + (Y - Y')^2  &\le (1+1/\alpha)R^2(\cos(\Theta) - \cos(\Theta'))^2 + (1 + \alpha)(R-R')^2\cos(\Theta')^2\\
    &\qquad+ (1+1/\alpha)R^2(\sin(\Theta) - \sin(\Theta'))^2 + (1+\alpha)(R-R')^2\sin(\Theta')^2\\
    &\le (2 + 2/\alpha)R^2(\Theta - \Theta')^2 + (1+\alpha)(R - R')^2,
\end{align*}
where we used that fact that $\sin(\cdot)$ and $\cos(\cdot)$ are 1-Lipschitz and that $\sin^2(\Theta') + \cos^2(\Theta') = 1$.

Now, restricting $\alpha \in (0, 1)$ we get that 
\begin{align*}
    \E[(X-X')^2 + (Y - Y')^2] \le \frac{4\E[R^2]}{\alpha}\E[(\Theta - \Theta')^2] + (1+\alpha)\E[(R-R')^2]
\end{align*}
using the independence of $R$ and $\Theta$ and the fact that $\Theta'$ is a deterministic function of $\Theta$ given the codebook $\{\theta_1, \ldots, \theta_k\}$. 

When $d = 1$, we call $\E[(X-X')^2 + (Y-Y')^2]$ to be $\text{error}_0$ since it is the expected error in the reconstruction at level 0 and similarly, we call $\E[(R-R')^2] = \text{error}_1$ since it is the error in level $1$. We also call $\E[(\Theta - \Theta')^2] = \text{quant}_1$ since it is the error of quantizing angles in that level. We therefore have
\begin{align*}
    \text{error}_0 \le \frac{4\E[R^2]}{\alpha} \cdot \text{quant}_1 + (1+\alpha) \cdot \text{error}_1.
\end{align*}

Given a vector $(X_1, \ldots, X_d)$, where each $X_i \sim N(0, 1)$, let $(X'_1, \ldots, X'_d)$ be the reconstructions using $t = \log_2(d)$-level quantization as described in the introduction. Extending the above definitions, we say that $\text{error}_i$ is the total expected error in the reconstructions of level $i$ coordinates in the quantization scheme. Using the above, inequality, we get
\begin{align*}
    \text{error}_0 \le \frac{4d}{\alpha} \cdot \text{quant}_1 +  \frac{4d}{\alpha}(1+\alpha) \cdot \text{quant}_2 + \frac{4d}{\alpha}(1+\alpha)^3 \cdot \text{quant}_2 + \cdots + \frac{4d}{\alpha}(1+\alpha)^{t} \cdot \text{quant}_t + (1+\alpha)^{t+1} \cdot \text{error}_{t+1},
\end{align*}
where we use the fact that $\E[X_1^2 + \cdots + X_d^2] = d$. Since we store the top-level radius exactly, we have $\text{error}_{t+1} = 0$ and therefore,
\begin{align*}
    \frac{\text{error}_0}{d} \le \frac{4}{\alpha}\left(\text{quant}_1 + (1+\alpha) \cdot \text{quant}_2 + \cdots + (1+\alpha)^t \cdot \text{quant}_t\right).
\end{align*}
For each level $i$, given $\varepsilon > 0$, we will now upper bound the size of the codebook for level $i$ so that $\text{quant}_i \le \varepsilon$. It is clear that a codebook $\{0, \sqrt{\varepsilon}, 2\sqrt{\varepsilon}, \ldots, 2\pi\}$ has a size $\lceil{2\pi/\sqrt{\varepsilon}}\rceil + 1$ and has $\text{quant}_i \le \varepsilon$ irrespective of the distribution of the angles. We would like to use the fact that the distribution of angles gets concentrated with increasing level $i$ to obtain better bounds on the size of the codebook. To that end, we prove the following lemma.
\begin{lemma}
    Let $X \in [0, \pi/2]$ be an arbitrary random variable with $\text{Var}(X) = \sigma^2$. Given $x$ and a set $S = \{x_1, \ldots, x_k\}$, define $d(x, S) = \min_{i \in [k]}|x_i - x|$. Define
    \begin{align*}
        \text{Var}_k(X) := \min_{S : |S| = k} E[d(X, S)^2].
    \end{align*}
    Given $\varepsilon > 0$, for $k = \Omega(\log(1/\sigma)/\sqrt{\varepsilon})$, we have $\text{Var}_k(X) \le \varepsilon \cdot \text{Var}_1(X) = \varepsilon \cdot \sigma^2$.
\end{lemma}
The lemma shows that as variance decreases, we can get tighter approximation bounds using the same value of $k$. The proof of this lemma is similar to that of Lemma 2 in \citet{bhattacharya2022k}.
\begin{proof}
Let $\mu = \E[X]$. Consider the interval $[\mu - \sigma, \mu + \sigma]$ and consider the points $S_1 = \{\mu, \mu\, \pm\, \varepsilon \sigma, \mu \pm 2\varepsilon\sigma, \ldots, \mu\, \pm\, \sigma\}$. Note that $|S_1| = 3 + 2/\varepsilon$. We have
\begin{align*}
    \E[d(X, S_1)^2\, \mid X \in [\mu - \sigma, \mu + \sigma]] \le \varepsilon^2\sigma^2
\end{align*}
since every point in the interval $[\mu - \sigma, \mu + \sigma]$ has some point in $S_1$ that is at most $\varepsilon\sigma$ away. 

Now define $S_2 = \{\mu\, \pm\, (1+\varepsilon)^{0}\sigma, \mu\, \pm\, (1+ \varepsilon)^1\sigma, \ldots, \}$ where the exponent $i$ extends until we have $\mu + (1+\varepsilon)^i\sigma \ge \pi/2$ and $\mu - (1+\varepsilon)^i \sigma \le 0$. Note that we have $|S_2| \le O(\log(1/\sigma)/\varepsilon)$. Suppose $|X-\mu| \in [(1+\varepsilon)^i\sigma, (1+\varepsilon)^{i+1}\sigma]$, then $d(X, S_2) \le \varepsilon(1+\varepsilon)^i\sigma \le \varepsilon|X-\mu|$. Now,
\begin{align*}
    \E[d(X, S_1 \cup S_2)^2] \le \E[\max(\varepsilon\sigma, \varepsilon|X-\mu|)^2] \le 2\varepsilon^2\sigma^2.
\end{align*}
Thus by putting $\varepsilon \coloneqq \sqrt{\varepsilon/2}$ above, if $k = \Omega(\log (1/\sigma)/\sqrt{\varepsilon})$, then $\text{Var}_k(X) \le \varepsilon\sigma^2$.
\end{proof}
Now we consider bounding $\text{quant}_i$ by picking an appropriate codebook for level $i$. For all $i > 0$, given a codebook $\mathcal{C} = \{\theta_1, \ldots \theta_{|\mathcal{C}|}\}$, we have $\text{quant}_i = \E[(\Theta - \text{round}(\Theta, \mathcal{C}))^2]$ where $\Theta$ is $\arctan(Y/X)$ with $Y^2, X^2 \sim \chi^2_{2^{i-1}}$. From the lemma above, when $i >1$, we have $\text{Var}(\Theta) \le C/(2^{i-1} - 1)$ and hence, the above lemma shows that there exists a codebook $\mathcal{C}$ of size $|\mathcal{C}| = \Theta(i/\sqrt{\varepsilon})$,
\begin{align}
    \text{quant}_i  \le \frac{\varepsilon}{2^{i-1}}.
\end{align}
If we use a codebook of size $O(1/\sqrt{\varepsilon})$ for level 1, we have $\text{quant}_1 \le \varepsilon$ and as described above, for $i \ge 1$, if we use a codebook of size $|\mathcal{C}| = \Theta(i/\sqrt{\varepsilon})$, then $\text{quant}_i \le \varepsilon/(2^{i-1})$ from which we obtain that
\begin{align*}
    \frac{\text{error}_0}{d} \le \frac{4}{\alpha}(2\varepsilon + \frac{1+\alpha}{\varepsilon} (2\varepsilon) + \left(\frac{1+\alpha}{2}\right)^2 (2\varepsilon) + \cdots ) \le O(\varepsilon)
\end{align*}
picking $\alpha = 1/2$.
Now we compute the number of bits necessary to store the quantized representation of a $d$-dimensional vector. In level $i$, we have $d/2^{i}$ independent samples of $\Theta$ to be quantized and therefore, we need to store $O(\frac{d}{2^{i}} \log(i/\sqrt{\varepsilon}))$ bits for indexing into the codebook for level $i$. Adding over all the levels, we store $O(d\log 1/\sqrt{\varepsilon})$ bits for representing the $d$-dimensional vector and therefore $O(\log 1/\varepsilon)$ bits per coordinate to achieve an average reconstruction error of $\varepsilon$.

\end{document}